\documentclass[acmsmall]{acmart}
\usepackage{subcaption} 
\AtBeginDocument{%
  }

\setcopyright{acmlicensed}
\copyrightyear{2018}
\acmYear{2018}
\acmDOI{XXXXXXX.XXXXXXX}

\acmJournal{JACM}
\acmVolume{37}
\acmNumber{4}
\acmArticle{111}
\acmMonth{8}

\begin{document}

\title{GRAM-MAMBA: Holistic Feature Alignment for Wireless Perception with Adaptive Low-Rank Compensation }

\author{Weiqi Yang}
\email{2023111372@mail.hfut.edu.cn}
\orcid{0009-0004-5409-1638}
\affiliation{%
  \institution{Hefei University of Technology}
  \city{Hefei}
  \state{Anhui}
  \country{China}
}

\author{Xu Zhou}
\affiliation{%
  \institution{Sun Yat-sen University}
  \city{Guangzhou}
  \state{Guangdong}
  \country{China}}
\email{zhoux369@mail2.sysu.edu.cn}

\author{Jingfu Guan}
\affiliation{%
 \institution{Hefei University of Technology}
 \city{Hefei}
 \state{Anhui}
 \country{China}}
\email{2024180208@mail.hfut.edu.cn}

\author{Hao Du}
\affiliation{%
  \institution{Jilin University}
  \city{Changchun}
  \state{Jilin}
  \country{China}
}
\email{duhao22@mails.jlu.edu.cn}

\author{Tianyu Bai}
\affiliation{%
 \institution{Durham University}
 \city{Durham}
 \country{UK}}
\email{Batiaoyu11@gmail.com}

\renewcommand{\shortauthors}{Yang et al.}

\begin{abstract}
  Multi-modal fusion is crucial for Internet of Things (IoT) perception, widely deployed in smart homes, intelligent transport, industrial automation, and healthcare. However, existing systems often face challenges: high model complexity hinders deployment in resource-constrained environments, unidirectional modal alignment neglects inter-modal relationships, and robustness suffers when sensor data is missing. These issues impede efficient and robust multimodal perception in real-world IoT settings. To overcome these limitations, we propose GRAM-MAMBA. This framework utilizes the linear-complexity Mamba model for efficient sensor time-series processing, combined with an optimized GRAM matrix strategy for pairwise alignment among modalities, addressing the shortcomings of traditional single-modality alignment. Inspired by Low-Rank Adaptation (LoRA), we introduce an adaptive low-rank layer compensation strategy to handle missing modalities post-training. This strategy freezes the pre-trained model core and irrelevant adaptive layers, fine-tuning only those related to available modalities and the fusion process. Extensive experiments validate GRAM-MAMBA's effectiveness. On the SPAWC2021 indoor positioning dataset, the pre-trained model shows lower error than baselines; adapting to missing modalities yields a 24.5\% performance boost by training <0.2\% of parameters. On the USC-HAD human activity recognition dataset, it achieves 93.55\% F1 and 93.81\% Overall Accuracy (OA), outperforming prior work; the update strategy increases F1 by 23\% while training <0.3\% of parameters. These results highlight GRAM-MAMBA's potential for achieving efficient and robust multimodal perception in resource-constrained environments.
\end{abstract}



\received{20 February 2007}
\received[revised]{12 March 2009}
\received[accepted]{5 June 2009}

\maketitle

\section{INTRODUCTION}
In recent years, the Internet of Things (IoT) has witnessed exponential growth fueled by advances in embedded systems and wireless communication technologies, with projections indicating over 30 billion connected devices globally by 2025 \cite{iot_growth_2025}. This proliferation has enabled pervasive sensing across diverse domains including smart cities \cite{smart_cities_iot}, precision agriculture \cite{precision_agriculture_iot}, and digital healthcare \cite{digital_healthcare_iot}. A particularly promising development has been the emergence of multimodal wireless perception systems that synergistically combine traditional sensors such as inertial measurement units (IMUs) \cite{imu_sensing} and computer vision \cite{computer_vision_sensing} with ubiquitous wireless signals including WiFi Channel State Information \cite{wifi_csi_sensing}, ultra-wideband \cite{uwb_sensing}, and millimeter-wave radar \cite{mmwave_radar_sensing}.These hybrid systems leverage the complementary strengths of different sensing modalities - where optical sensors provide rich spatial information \cite{computer_vision_sensing}, inertial sensors offer precise motion tracking \cite{imu_sensing}, and RF signals enable device-free sensing through walls and obstructions \cite{wireless_sensing_survey}. Recent studies have demonstrated that such multimodal fusion can improve activity recognition accuracy by up to 27\% compared to unimodal approaches \cite{multimodal_fusion_survey}, while simultaneously reducing false positives in applications like fall detection \cite{sensor_fault_detection} and intrusion monitoring \cite{wireless_sensing_survey}. The integration of these technologies is driving innovation across multiple sectors including intelligent transportation systems \cite{multimodal_fusion_survey}, contactless health monitoring \cite{digital_healthcare_iot}, and next-generation human-computer interaction \cite{computer_vision_sensing}.

Despite this potential, achieving efficient and robust multimodal perception in practical IoT settings is hindered by significant hurdles. First, many state-of-the-art fusion models, particularly those based on Transformers, suffer from high computational complexity \cite{transformer_complexity} (often quadratic with sequence length), rendering them unsuitable for deployment on resource-constrained edge devices typical in IoT environments with limited power, memory, and processing capabilities \cite{edge_computing_iot}. Recent theoretical analysis \cite{computational_complexity_transformers} has shown that traditional attention mechanisms scale poorly with sequence length, necessitating the development of more efficient architectures like structured state spaces \cite{s4_original} and linear attention mechanisms \cite{linear_attention}. Second, existing fusion techniques frequently employ suboptimal feature alignment strategies. Simple concatenation or unidirectional alignment—where diverse modalities are mapped to a single reference source—often fails to capture the intricate pairwise or holistic inter-modal correlations and temporal dependencies vital for comprehensive understanding \cite{multimodal_fusion_survey,deep_multimodal_learning}. Theoretical work on multimodal learning \cite{modality_dropout_theory} has demonstrated that heterogeneous features require sophisticated alignment mechanisms to realize their full potential, thus limiting representational power and fusion efficacy when using naive approaches. Third, robustness against missing modalities remains a critical challenge. Real-world IoT systems inevitably face sensor failures or data dropouts, causing models trained on complete data to suffer significant performance degradation \cite{sensor_fault_detection}. While existing methods attempt to maintain robustness in the face of missing modalities, often resorting to data imputation strategies, these approaches have notable disadvantages. For example, deep learning-based imputation techniques can suffer from cumulative errors propagating through the generated data, while simpler methods like interpolation require prior knowledge of surrounding data points, limiting their applicability. Furthermore, many imputation strategies require pre-computation and may not be sufficiently adaptive to dynamic data dropout scenarios. This highlights the urgent need for robust and efficiently adaptive perception systems capable of handling modality absence without relying solely on pre-computed imputation or costly retraining protocols. Current approaches in multimodal IoT perception often fall short of simultaneously addressing these critical challenges. Models optimized for efficiency frequently sacrifice sophisticated fusion and alignment, relying on simpler techniques that may underperform \cite{computational_complexity_transformers}. Conversely, methods achieving strong inter-modal alignment through complex mechanisms like cross-attention often incur prohibitive computational costs for edge devices \cite{edge_computing_iot}. Moreover, existing techniques for ensuring robustness against missing data typically necessitate either computationally intensive retraining protocols, rely on imputation methods with inherent limitations, or lack methods for truly parameter-efficient adaptation based on pre-trained models \cite{parameter_efficient_finetuning}. This urgently demands a unified framework enabling concurrent computational efficiency, holistic feature alignment, and adaptive robustness based pre-trained model for multi-modal wireless perception.

To address these limitations, this paper proposes GRAM-MAMBA, a novel framework designed for efficient, robust, and holistically aligned multimodal wireless perception.GRAM-MAMBA leverages the Mamba architecture, a recent state space model renowned for its linear complexity scaling with sequence length, making it highly suitable for processing long sensor time series data efficiently on resource-constrained devices. We integrate Mamba with an optimized GRAM matrix-based alignment strategy. This explicitly models and aligns features between all pairs of modalities, capturing their rich interdependencies holistically, overcoming the limitations of unidirectional alignment.Inspired by the parameter-efficient fine-tuning technique Low-Rank Adaptation (LoRA), we introduce an Adaptive Low-Rank Compensation mechanism. This allows the pre-trained GRAM-MAMBA model to adapt efficiently to scenarios with missing sensor modalities by updating only small, low-rank adaptive layers, freezing the bulk of the model parameters, and avoiding costly retraining.Crucially, this parameter-efficient adaptation directly learns to compensate for missing information within the model itself, thereby overcoming the limitations inherent in traditional data imputation approaches, such as the potential for cumulative error propagation, the requirement for unavailable contextual data (like in interpolation), or the need for computationally intensive preprocessing steps.

GRAM-MAMBA is able to perform highly accurate Human Activity Recognition (HAR) and Indoor Positioning efficiently using multimodal wireless data, significantly improving robustness when sensors are missing. For HAR on the USC-HAD dataset, the pre-trained GRAM-MAMBA achieves 93.55\% F1 score and 93.81\% OA, outperforming prior methods. Critically, when adapting to missing modalities, its parameter-efficient update (less than 0.3\% of parameters) leads to a 23\% F1 score improvement, demonstrating superior robustness and performance recovery compared to typical data imputation strategies, while drastically reducing adaptation costs versus full retraining. Similarly, for indoor positioning on SPAWC2021, it yields lower errors than baselines, and its adaptation to missing modalities achieves a 24.5\% performance boost by updating less than 0.2\% parameters.

In brief, the contribution of the paper includes:
\begin{itemize}
    \item We propose GRAM-MAMBA, a novel architecture combining the efficiency of Mamba with a GRAM matrix for holistic, pair-wise multimodal feature alignment in wireless perception.
    \item A novel Adaptive Low-Rank Compensation strategy (inspired by LoRA) enables robust, parameter-efficient adaptation to missing modalities with minimal (<0.3\%) parameter updates, offering an effective alternative to data imputation techniques.
    \item Demonstration of GRAM-MAMBA's effectiveness through extensive experiments on standard benchmark datasets for indoor positioning (SPAWC2021) and human activity recognition (USC-HAD). Achieving state-of-the-art or competitive performance while significantly improving efficiency and demonstrating remarkable performance recovery during modality dropout scenarios.
\end{itemize}

\section{RELATED WORK}

When researchers seek to develop robust and accurate wireless perception systems for applications like activity recognition or localization within complex IoT environments, leveraging multi-modal fusion, efficient temporal modeling, and data imputation often become the primary strategies for achieving reliable performance under diverse conditions.

\subsection{Multi-modal Fusion Techniques in  Wireless Perception}

Effectively integrating information from diverse sources through multimodal fusion is fundamental to enhancing wireless perception systems like HAR and positioning. Research in this area explores various strategies for feature integration and alignment, aiming to maximize performance gains over unimodal approaches.

Chen et al. \cite{xfi_framework} introduced the X-Fi framework, which enforces cross-modal invariance by aligning features from modalities like mmWave radar and RGB-D to the latent space derived from IMU data. This alignment, achieved using contrastive loss to maximize feature similarity for identical actions across modalities, maps all sources to the IMU's reference feature space, yielding state-of-the-art results in human pose estimation and activity recognition. Similarly, Ouyang et al. \cite{cosmo_system} proposed the Cosmo system, incorporating a quality-guided attention mechanism. This dynamically assesses the relevance of auxiliary modalities (e.g., gyroscope, ambient sound) relative to accelerometer signals, re-weighting features from lower-quality sources to better match the accelerometer's temporal feature distribution. Training leverages both cloud-based unlabeled data and limited edge-labeled data, significantly enhancing activity recognition accuracy with heterogeneous inputs. Addressing efficient data representation, Zheng et al. \cite{information_bottleneck_fusion} utilized the Information Bottleneck principle to compress event camera streams, preserving spatio-temporal features most correlated with skeleton motion trajectories before fusing them with skeleton keypoint data, demonstrating excellent accuracy and energy efficiency. 

Theoretical foundations for multimodal alignment have been established through canonical correlation analysis \cite{canonical_correlation}, which provides a mathematical framework for finding linear relationships between different modalities. Recent advances in deep multimodal learning \cite{deep_multimodal_learning} have extended these concepts to non-linear mappings, enabling more sophisticated feature alignment strategies. However, while effective, these approaches predominantly focus on aligning multiple modalities towards a single reference sensor space (e.g., IMU or accelerometer), potentially overlooking richer, direct pairwise relationships between all modalities. The mathematical properties of Gram matrices \cite{gram_matrix_applications} offer a principled approach to capture such holistic relationships, which forms the theoretical basis for our proposed alignment strategy.

\subsection{Efficient Sequence Modeling for Sensor Time Series}

Processing sequential sensor data effectively is critical in IoT. While recurrent networks and later Transformers significantly advanced capabilities, their respective limitations, particularly the computational scaling of self-attention, have motivated the exploration of more novel
 neural network architectures .

Sophisticated sequence models like Transformers and recurrent networks are increasingly applied to wireless and sensor time series data. Wang et al. \cite{wir_transformer} pioneered Wir-transformer for wireless interference identification, using self-attention to outperform CNN/LSTM models, particularly in low SNR scenarios. However, the quadratic complexity of self-attention mechanisms \cite{attention_theory} poses significant challenges for long sequences typical in sensor data. Atik Faysal et al. \cite{transformer_lstm_hybrid} proposed a Transformer-LSTM hybrid architecture, effectively combining global context capture with sequential dependency modeling, showing benefits in few-shot learning scenarios. Investigating recurrent models, Yang et al. \cite{bilstm_gesture_recognition} demonstrated BiLSTM's superiority over LSTM for longer sequences in WiFi CSI-based gesture recognition via their SenseFi library, highlighting the importance of bidirectional processing for wireless signal analysis \cite{wireless_signal_processing}. 

Recent advances in structured state space models \cite{state_space_models,s4_original} have addressed the computational limitations of traditional sequence models by achieving linear complexity while maintaining the modeling capacity of Transformers. The emergence of Mamba \cite{mamba_original} represents a significant breakthrough in this direction, offering selective state spaces that can efficiently process long sequences with linear computational complexity. Furthermore, lightweight approaches are explored; Deng et al. \cite{lhar_knowledge_distillation} developed the LHAR framework using knowledge distillation to create efficient yet effective student models for on-device perception. While techniques like knowledge distillation can yield efficient student models, they often rely on computationally expensive pre-training of large teacher models. Furthermore, the quadratic complexity of standard Transformers often remains a direct bottleneck for resource-constrained applications, motivating the exploration of more scalable architectures like linear attention mechanisms \cite{linear_attention}.

\subsection{Robustness and Adaptation in Multimodal Systems}

Ensuring reliable performance when some sensor modalities are unavailable is critical for real-world IoT Perception systems. Currently, data imputation techniques are prevalent methods employed to handle missing data. This section reviews these and other existing approaches for adapting multi-modal models, highlighting the associated challenges related to retraining cost and adaptation efficiency.

Adiba et al. \cite{filsm_federated}, through their work on FILSM, specifically underscore the detrimental effects of modality absence on efficiency and scalability, particularly within federated learning contexts \cite{federated_multimodal} where centralized data pooling is infeasible. Furthermore, common real-world issues such as device malfunctions, unstable network connectivity, or user privacy constraints frequently result in incomplete multimodal data streams, intensifying the need for efficient adaptation strategies. Theoretical analysis of multimodal systems \cite{modality_dropout_theory} has shown that the performance degradation due to missing modalities can be severe, particularly when the model is not designed with robustness in mind. 

Existing approaches often focus on data completion techniques prior to model input. For instance, Q.Le et al. \cite{mfcpl_multimodal} explored data imputation and reconstruction methods within their MFCPL framework to handle scenarios with significant modality loss. Similarly, Nwamaka et al. \cite{vae_imputation} discussed the impact of missing sensor data on system calibration and proposed using generative models, such as Variational Autoencoders (VAEs), for imputation. Traditional statistical methods like MICE \cite{mice_imputation} and MissForest \cite{missforest_imputation} have also been adapted for sensor data imputation. However, a significant drawback common to these methods is the frequent requirement to retrain the entire perception model to incorporate the reconstructed data or effectively adapt to the altered input distribution, a process which is computationally expensive and often impractical for dynamic, real-time deployments. This highlights the need for more parameter-efficient adaptation mechanisms \cite{parameter_efficient_finetuning} that can leverage pre-trained knowledge without full retraining, such as the LoRA approach \cite{lora_original} which has shown promise in natural language processing tasks.

\section{GRAM-MAMBA}

This section details our proposed framework, GRAM-MAMBA, designed to overcome the identified limitations in multimodal perception concerning inter-modal correspondence, computational complexity, and adaptation to missing modalities. We begin by providing a precise definition of the problem context. Subsequently, we introduce the overall architecture of GRAM-MAMBA and highlight its key design aspects aimed at improving fusion effectiveness and computational efficiency. Finally, we provide detailed explanations to each constituent module of GRAM-MAMBA.
\subsection{Problem Statement}

Current multimodal perception systems, processing diverse inputs $x = \{x_1, \dots, x_N\}$ from modalities $m_i \in M$, face critical challenges: 1) Many prevalent fusion architectures learn encoders $enc_i: x_i \mapsto z_i \in \mathbb{R}^k$ but often employ an \textbf{anchor-based alignment strategy}, designating a specific anchor modality $m_a$ and primarily optimizing the alignment of other modality representations $z_j$ (where $j \neq a$) towards the anchor $z_a$. This focus inadvertently neglects the explicit modeling of \textbf{pairwise correspondence between non-anchor modalities} (i.e., ensuring high $\text{Sim}(z_j, z_k)$ for $j, k \neq a$), thereby failing to capture the full spectrum of inter-modal relationships and potentially limiting fusion performance. 2) The common reliance on Transformer models within these architectures introduces \textbf{$\mathcal{O}(L^2)$ computational complexity} relative to sequence length $L$, hindering efficiency and scalability, especially with long sequential data or in resource-limited settings. 3) Adapting large pre-trained models (parameters $\Theta$) when faced with \textbf{modal unavailability} during updates (using only $x' = \{x_i \mid m_i \in M', M' \subset M\}$) typically requires computationally expensive \textbf{full retraining of $\Theta$} or intricate data imputation for missing $\{x_j \mid m_j \notin M'\}$, lacking effective \textbf{parameter-efficient fine-tuning (PEFT)} methods that learn only a minimal set of parameters $\Delta\Theta$ ($|\Delta\Theta| \ll |\Theta|$) using just the incomplete data $x'$.

\subsection{Framework Overview}
\begin{figure}
    \centering
    \includegraphics[width=1\linewidth]{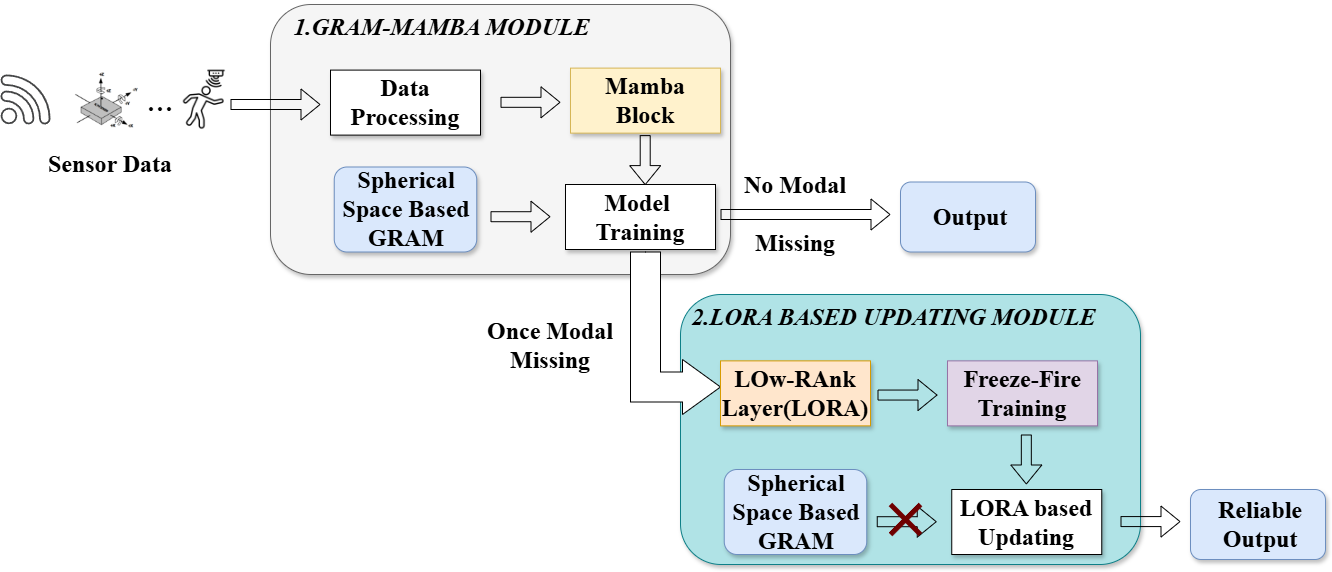}
    \caption{System Overview}
    \label{fig:enter-label}
\end{figure}

As a solution, we propose GRAM-MAMBA, a novel multi-modal perception framework designed for enhancing non-anchor inter-modal correspondence and achieving efficient sequence modeling, incorporating a LoRA-inspired strategy for parameter-efficient adaptation to missing modalities. We present an overview of our proposed framework covering three core contributions: (1) Utilizing GRAM matrices to explicitly enforce pairwise correspondence between non-anchor modality representations sharing the same label; (2) Employing the Mamba architecture for effective sequential data modeling with linear computational complexity; and (3) Integrating low-rank adaptation layers for efficient fine-tuning using incomplete modal data during model updates. The design and details of each component of our proposed framework are given below in the following sections.

\subsection{Model Training Module}
\begin{figure}[htbp] 
    \centering 

    \begin{subfigure}[b]{0.48\textwidth} 
        \centering
        \includegraphics[width=\linewidth]{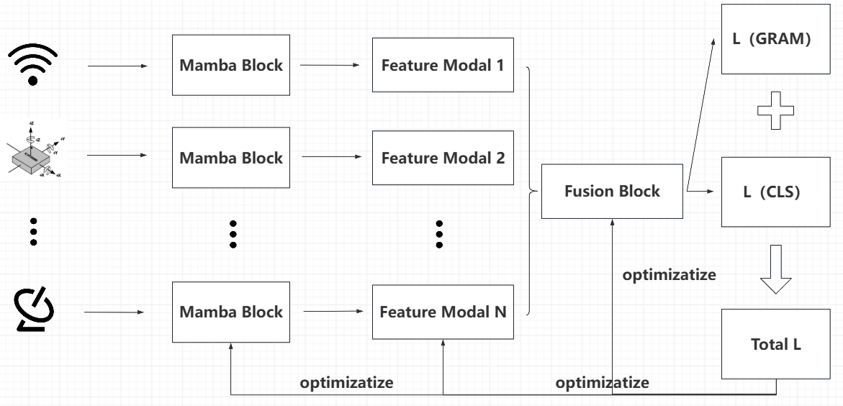} 
        \caption{Pre-training or Full Modal Training Phase}
        \label{fig:phase_a} 
    \end{subfigure}
    \hfill 
    \begin{subfigure}[b]{0.48\textwidth} 
        \centering
        \includegraphics[width=\linewidth]{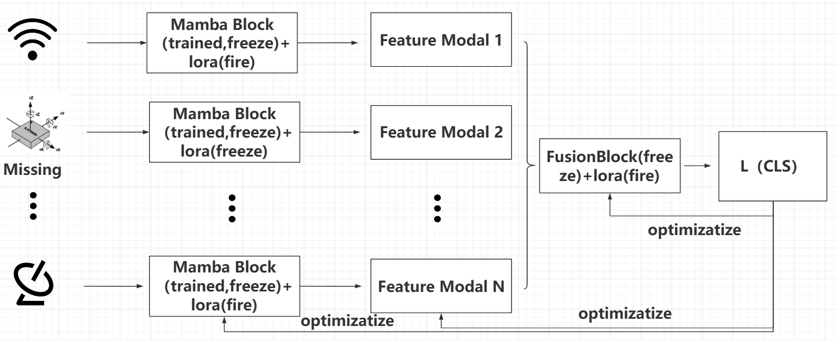} 
        \caption{Adaptation Phase with Missing Modalities using LoRA}
        \label{fig:phase_b} 
    \end{subfigure}

    \caption{Overview of the GRAM-MAMBA framework phases. (a) Shows the training with all modalities, optimizing both GRAM and Classification losses. (b) Illustrates the parameter-efficient adaptation using LoRA when some modalities are missing, freezing base parameters and only training LoRA layers.}
    \label{fig:framework_overview} 
\end{figure}

Let the input consist of $N$ distinct modalities $M = \{m_1, \dots, m_N\}$, encompassing diverse sensor types like Wi-Fi channel state information, IMU readings, or Radar signatures. These are typically captured as time series data, $x = \{x_1, \dots, x_N\}$, where each $x_i$ represents a sequence of measurements of length $L_i$. Each modality stream $x_i$ is independently processed by a dedicated Mamba block \cite{mamba_original}, a structured State Space Model (SSM) architecture specifically designed for efficient modeling of long sequences. Mamba builds upon the foundation of continuous-time SSMs ($h'(t) = Fh(t) + Gx(t)$, $y(t) = Jh(t)$), which are discretized using methods like Zero-Order Hold (ZOH) with a step size $\Delta$ to yield parameters $\bar{F} = \exp(\Delta F)$ and $\bar{G} = (\Delta F)^{-1}(\exp(\Delta F) - I) \cdot \Delta G$. The core Mamba computation involves a selective scan mechanism operating on the discrete recurrence relation:
\begin{equation}
 h_t = \bar{F} h_{t-1} + \bar{G} x_t
\end{equation}
where $h_t$ represents the hidden state at time step $t$. This selection mechanism allows Mamba to focus on relevant information across the sequence, enabling efficient modeling with linear complexity $O(L)$ relative to sequence length $L$, a crucial advantage over the $O(L^2)$ complexity of standard Transformers, especially in resource-constrained IoT settings. The $i$-th Mamba block, parameterized by $\Theta_{\text{Mamba}, i}$, maps the input sequence $x_i$ to a final raw feature vector $\tilde{z}_i \in \mathbb{R}^k$, typically derived from the final hidden state or a pooling operation over the sequence of states.

To focus on the directional similarity and inter-modal correlation, independent of feature magnitudes, each raw feature vector $\tilde{z}_i$ is L2-normalized, projecting it onto the unit hypersphere: $z_i = \tilde{z}_i / \|\tilde{z}_i\|_2$. These normalized features $Z = \{z_1, \dots, z_N\}$ form the basis for our holistic inter-modal alignment strategy using a Gram matrix $G \in \mathbb{R}^{N \times N}$. Due to the normalization, the pairwise similarity simplifies to the inner product: $G_{ij} = z_i \cdot z_j$, with $G_{ii}=1$. The primary objective is to enforce strong alignment ($G_{ij} \to 1$) between features $z_i, z_j$ derived from different modalities but corresponding to the same underlying semantic class (e.g., the same activity). This explicitly models all pairwise correlations, unlike anchor-based methods that only align secondary modalities to a single reference. The mechanism to encourage this alignment relies on the properties of the Gram matrix determinant, as formalized below:

\begin{theorem}
Let $G$ be the $N \times N$ Gram matrix constructed from L2-normalized vectors $Z = \{z_1, \dots, z_N\}$, where $G_{ij} = z_i \cdot z_j$. As $G_{ij} \to 1$ for all $1 \le i, j \le N$, the determinant of $G$ approaches zero, i.e., $\det(G) \to 0$.
\end{theorem}

\begin{proof}
Consider the case where the Gram matrix $G$ approaches the matrix of all ones. This occurs when the cosine similarity $G_{ij}$ between all pairs of distinct normalized vectors ($i \neq j$) approaches 1. Let us analyze a specific structure where all off-diagonal elements are equal, $G_{ij} = c$ for $i \neq j$, and $G_{ii} = 1$ due to normalization. The determinant of such an $N \times N$ matrix is given by \cite{matrix_analysis}:
\begin{equation}
 \det(G) = (1 - c)^{N-1} (1 + (N - 1)c)
\end{equation}
As the features become perfectly aligned, the similarity $c = G_{ij}$ approaches 1 for $i \neq j$. Taking the limit as $c \to 1$:
\begin{equation}
 \lim_{c \to 1} \det(G) = \lim_{c \to 1} (1 - c)^{N-1} (1 + (N - 1)c)
\end{equation}
Since $N > 1$ (for multimodal scenarios), the term $(1 - c)^{N-1} \to 0$. The term $(1 + (N - 1)c) \to (1 + N - 1) = N$. Thus, the limit is $0 \times N = 0$.
While the actual Gram matrix $G$ resulting from the feature vectors may not have perfectly identical off-diagonal elements, the principle holds: as all $G_{ij} \to 1$, the matrix $G$ approaches the rank-1 matrix of all ones, whose determinant is zero. Therefore, minimizing $\det(G)$ drives all pairwise similarities $G_{ij}$ towards 1.
\end{proof}

Leveraging Theorem 1, we define the alignment loss term as $L_{\text{GRAM}}(G) = \det(G)$. Minimizing this loss directly encourages the desired alignment where $G_{ij} \to 1$. Geometrically, this minimization corresponds to collapsing the volume $V = \sqrt{\det(G)}$ of the parallelepiped spanned by the feature vectors $z_1, \dots, z_N$, driving them towards semantic agreement represented by collinearity in the feature space.

Subsequent to feature extraction and alignment regularization, the normalized modality features $Z$ are typically aggregated via concatenation, $z_{\text{concat}} = [z_1; \dots; z_N]$. This aggregated representation is then processed by a fusion block (e.g., an MLP parameterized by $\Theta_{\text{Fusion}}$) and a final classification layer (parameterized by $\Theta_{\text{CLS}}$) using Softmax to produce the class probability distribution $\hat{y}$. The entire network architecture, encompassing all Mamba encoders and the fusion/classification head, with parameters $\Theta = \{\Theta_{\text{Mamba}, i}, \Theta_{\text{Fusion}}, \Theta_{\text{CLS}}\}$, is trained end-to-end. The optimization minimizes a composite loss function that balances the primary task objective (e.g., cross-entropy $L_{\text{CLS}}$) with the inter-modal alignment pressure:
\begin{equation}
 L_{\text{total}} = L_{\text{CLS}}(\hat{y}, y) + \beta L_{\text{GRAM}}(G)
\end{equation}
Here, $y$ represents the ground-truth label, and the hyperparameter $\beta$ controls the relative contribution of the GRAM determinant loss, ensuring that alignment enhancement does not unduly compromise classification accuracy. This integrated framework synergistically combines the sequence modeling efficiency of Mamba with the explicit, holistic alignment enforced by the GRAM determinant loss, aiming for accurate, efficient, and robust multimodal perception suitable for challenging IoT applications.

\subsection{Model Updating Module}

While the GRAM-MAMBA framework described previously is trained assuming the availability of all $N$ modalities, real-world deployment necessitates robustness against sensor failures or data dropouts, resulting in scenarios where only a subset of modalities $M' \subset M$ is available. To address this challenge efficiently without resorting to computationally expensive full retraining for every possible subset $M'$, we introduce an Adaptive Low-Rank Compensation mechanism inspired by Parameter-Efficient Fine-Tuning (PEFT) techniques, specifically Low-Rank Adaptation (LoRA) \cite{lora_original}.

LoRA allows for efficient adaptation of large pre-trained models by freezing the vast majority of the original model parameters ($\Theta$) and introducing a small number of trainable parameters ($\Delta \Theta$, where $|\Delta \Theta| \ll |\Theta|$). For a given pre-trained layer with weight matrix $W_0 \in \mathbb{R}^{d \times k}$, LoRA modifies the forward pass $h = W_0 x$ by adding a low-rank update:
\begin{equation}
    h = W_0 x + \Delta W x = W_0 x + B A x
\end{equation}
Here, $\Delta W = B A$ represents the low-rank decomposition, where $B \in \mathbb{R}^{d \times r}$ and $A \in \mathbb{R}^{r \times k}$ are the trainable low-rank matrices, and the rank $r \ll \min(d, k)$. During adaptation, only the parameters within $A$ and $B$ are updated, while the original weights $W_0$ remain frozen. This concept is illustrated conceptually in Figure~\ref{fig:phase_b}.

In the context of GRAM-MAMBA, we strategically inject these trainable LoRA modules into the pre-trained architecture. As depicted in Figure~\ref{fig:phase_a}, a distinct LoRA adapter, parameterized by $\Delta \Theta_{\text{Mamba}, i}$, is inserted after each modality-specific Mamba encoder ($\Theta_{\text{Mamba}, i}$). Additionally, another LoRA adapter, parameterized by $\Delta \Theta_{\text{Fusion}}$, is incorporated following the multimodal fusion block ($\Theta_{\text{Fusion}}$). The initial pre-training optimizes the base parameters $\Theta = \{\Theta_{\text{Mamba}, 1..N}, \Theta_{\text{Fusion}}, \Theta_{\text{CLS}}\}$ along with the $L_{\text{GRAM}}$ objective.

When faced with a scenario where only modalities $M'$ are present, the adaptation proceeds as follows:
\begin{enumerate}
    \item \textbf{Parameter Freezing:} The original pre-trained parameters $\Theta$ of the entire network are frozen. Furthermore, the LoRA adapters corresponding to the \textit{missing} modalities, i.e., $\{\Delta \Theta_{\text{Mamba}, j} | m_j \notin M'\}$, are also frozen (or effectively bypassed).
    \item \textbf{Parameter Training:} Only the LoRA adapters associated with the \textit{available} modalities, $\{\Delta \Theta_{\text{Mamba}, i} | m_i \in M'\}$, and the LoRA adapter for the fusion block, $\Delta \Theta_{\text{Fusion}}$, are made trainable.
    \item \textbf{Fine-tuning Objective:} The model is fine-tuned using only the data from the available modalities $x' = \{x_i | m_i \in M'\}$. The optimization objective during this phase typically focuses on the primary task loss, $L_{\text{CLS}}$, calculated based on the final prediction derived from the available inputs processed through the frozen base model and the active, trainable LoRA layers. The $L_{\text{GRAM}}$ term, which relies on comparing all modalities, is generally not applicable or requires modification during this adaptation phase.
\end{enumerate}
This Adaptive Low-Rank Compensation strategy enables the pre-trained GRAM-MAMBA model to rapidly and efficiently adjust its internal representations and fusion logic to compensate for the missing information.

\section{Evaluation Methodology}

In this section, we introduce the public real-world datasets used for our experiments and the baselines for comparison.
\subsection{Dataset}

We conduct experiments on two public real-world wireless perception datasets.

\textbf{USC-HAD dataset:} The USC-HAD dataset \cite{usc_had_dataset} was collected from 14 subjects (7 female, 7 male) with the consideration of gender, age, height, and weight. This diversity enriches the dataset, ensuring its insights can be woven into a broader understanding of human movement across different
 demographics. Subjects wore the mobile phone at their front right hip, immersing themselves in the familiar rhythms of daily life. The dataset was gathered in natural settings rather than a controlled lab, reflecting authentic movement patterns. The data collection process spanned six hours per subject. The dataset consists of 12 activities including Walking right, Walking left, Walking forward, Walking upstairs, Walking downstairs, Jumping, Standing, Sitting, Sleeping,Running forward, Elevator up, and Elevator down. The sampling rate is 100 Hz, which offers a high-resolution view of human dynamics.

\textbf{PAMAP2 DATASET:}The PAMAP2 dataset \cite{pamap2_dataset} was collected from 9 subjects (5 male, 4 female) with diverse ages and BMI. This variety ensures broad applicability. Subjects wore sensors on the chest, wrist, and ankle, capturing activities in real-world settings. Each subject participated for about 2.5 hours.The dataset includes 18 activities like sitting, walking, cycling, and sports. The high 100 Hz sampling rate provides detailed motion data. Heart rate data at 1 Hz adds physiological insights. This comprehensive dataset supports robust activity recognition research.

\subsection{Baselines}

To comprehensively assess our proposed method, we conduct two sets of comparisons. First, the model's overall performance is benchmarked against six leading baseline methods using complete multi-modal data. Second, the robustness and efficiency of our adaptation mechanism for handling missing data are evaluated against nine data imputation strategies, which include five deep learning methods and four interpolation methods.

\textbf{Model's overall performance :}  OCluDAL \cite{ocludal_method} proposes a human activity recognition method based on dynamic active learning. By dynamically selecting unlabeled samples and identifying new activity categories, it significantly reduces the annotation cost and simultaneously improves the performance of activity recognition. ConvLSTM \cite{convlstm_attention} proposes a neural network architecture that combines deep convolutional LSTM (ConvLSTM) and self-attention mechanisms, which is used to accurately decode human activities through wearable sensors. LAG \cite{lag_method} proposes a local and global alignment (LAG) method for generalizing human activity recognition (HAR) based on sensors. It alleviates the distribution offset problem by learning the domain-invariant features between the training set and the test set, thereby enhancing the generalization ability of the model. 1D-CNN-LSTM: proposes a human activity recognition method that combines one-dimensional convolutional neural networks (1D CNN) and long Short-Term memory networks (LSTM). TCN-Attention \cite{tcn_attention} combines the Temporal Convolutional Network (TCN) and the attention mechanism. Through TCN, multi-scale temporal series features are extracted and the attention mechanism is utilized to dynamically focus on key information, significantly enhancing the modeling ability of time series tasks. Moreover, to fully verify the effectiveness of the proposed GRAM-MAMBA method, comparative experiments were also conducted against representative classical architectures, such as ViT \cite{vit_original}, DualViT, and Cait \cite{cait}.

\textbf{Handling missing data:} VAE \cite{vae_imputation} imputes missing values by learning a probabilistic latent representation of the data and generating plausible completions using its decoder network. NNRW estimates missing values by utilizing values from the most similar complete data points (nearest neighbors), often through regression or weighted averaging. MISSFOREST \cite{missforest_imputation} iteratively predicts missing values for each variable using Random Forest models trained on the observed portions of the data until the imputed dataset stabilizes. MICE \cite{mice_imputation} imputes missing data iteratively by fitting conditional regression models for each variable with missingness based on the other variables. ImputeFormer \cite{imputeformer} integrates low-rank structural priors into a Transformer architecture via projected temporal attention and embedded spatial attention, combined with a Fourier sparsity loss, to perform generalizable spatiotemporal data imputation. Moreover, to fully verify the effectiveness of the proposed adaptive compensation method, comparative experiments were also conducted with the classical interpolation methods (cubic, Lagrange, Akima and PCHIP).

\section{Evaluation Results}

\subsection{Model's overall performance}
Table 1 presents the experimental results comparing our proposed GRAM-MAMBA model against several baseline methods on the USC-HAD dataset for the Human Activity Recognition (HAR) task, assuming complete modal availability during evaluation. Performance is primarily evaluated using F1-score and Overall Accuracy (OA), with higher values indicating better performance.Meanwhile, compared with a host of classic architectures such as VIT, VIM, and CAIT, our approach also has significant advantages. Our GRAM-MAMBA model demonstrates superior performance across all evaluation metrics, significantly outperforming these HAR-focused baselines.

\begin{table}[htbp] 
    \centering 
    \caption{Performance comparison of GRAM-MAMBA against Vision Transformer (ViT) variants and other Mamba-based architectures on the HAR task.}
    \label{tab:comparison_vit_mamba} 
    \begin{tabular}{lcc}
        \toprule 
        \textbf{Model} & \textbf{OA (\%)} & \textbf{F1 (\%)} \\
        \midrule 
        DeepViT       & 87.87 & 84.14 \\
        ViT           & 89.55 & 85.96 \\
        SwimT         & 91.44 & 88.46 \\
        VIM           & 91.65 & 88.80 \\
        VAN           & 91.17 & 88.00 \\
        Cait          & 79.80 & 74.58 \\ 
        WaveVIT       & 91.61 & 88.62 \\
        DualVIT       & 91.78 & 89.04 \\
        MambaOut      & 91.13 & 87.89 \\
        ActivityMamba & 91.78 & 89.13 \\
        \midrule 
        \textbf{GRAM-Mamba} & \textbf{93.81} & \textbf{93.55} \\ 
        \bottomrule 
    \end{tabular}
\end{table}

\begin{table}[htbp] 
    \centering 
    \caption{Performance comparison of GRAM-MAMBA against HAR-focused baseline methods on the USC-HAD dataset.}
    \label{tab:comparison_har} 
    \begin{tabular}{lcc}
        \toprule 
        \textbf{Model} & \textbf{OA (\%)} & \textbf{F1 (\%)} \\
        \midrule 
        OCLuDAL                 & 91.70          & -              \\
        ConvLSTM                & 90.88          & -              \\
        LAG                     & 84.40          & -              \\
        1D-CNN-LSTM             & -              & 74.62          \\
        TCN-Attention           & 93.17          & 88.36          \\ 
        ActivityMamba           & 91.78          & 89.13          \\ 
        \midrule 
        \textbf{GRAM-Mamba}     & \textbf{93.81} & \textbf{93.55} \\ 
        \bottomrule 
    \end{tabular}
\end{table}

In summary, GRAM-MAMBA's superior performance in this standard HAR setting (with complete data) stems from its effective integration of the efficient Mamba backbone for capturing temporal dynamics and the novel GRAM matrix optimization that enforces holistic alignment across all available sensor modalities, leading to more discriminative and well-fused representations for accurate activity recognition.
\subsection{Handling missing data}
\begin{table}[htbp] 
    \centering 
    \caption{Performance of GRAM-MAMBA under different modality missing conditions using Adaptive Low-Rank Compensation.}
    \label{tab:adaptation_results} 
    \begin{tabular}{lccccc}
        \toprule 
        \textbf{Condition} & \textbf{OA (\%)} & \textbf{F1-score (\%)} & \textbf{Parameters} & \textbf{F1 $\uparrow$ (\%)} & \textbf{ACC $\uparrow$ (\%)} \\
        \midrule 
        Both normal & 93.93 & 93.92 & 320397 & - & - \\ 
        GYO Missing & 63.08 & 52.07 & 768 (\textbf{0.2397\%}) & 71.94 (19) & 77.64 (14) \\ 
        ACC Missing & 62.20 & 45.32 & 768 (\textbf{0.2397\%}) & 71.30 (26) & 71.24 (9)  \\ 
        \bottomrule 
    \end{tabular}
\end{table}
\begin{table}[htbp] 
    \centering 
    \caption{Performance comparison of different imputation and interpolation methods against the proposed adaptive approach ('Ours') under the 'ACC Missing' condition.}
    \label{tab:imputation_comparison_acc_missing} 
    \begin{tabular}{lcc}
        \toprule 
        \textbf{Method} & \textbf{OA (\%)} & \textbf{F1 (\%)} \\
        \midrule 
        \multicolumn{3}{l}{\textit{Advanced Imputation Methods}} \\ 
        VAE             & 74.89 & 63.27 \\
        NNRW            & 65.00 & 50.39 \\
        MISS FOREST     & 54.96 & 46.10 \\
        MICE            & 54.97 & 46.08 \\
        Impute Former   & 57.26 & 40.35 \\
        \midrule 
        \multicolumn{3}{l}{\textit{Classical Interpolation Methods}} \\ 
        cubic           & 67.79 & 60.47 \\
        lagrange        & 57.53 & 48.94 \\
        Akima           & 64.35 & 49.40 \\
        PCHIP           & 73.15 & 66.46 \\
        \midrule 
        \textbf{Ours (LoRA Adapt)} & \textbf{76.20} & \textbf{66.54} \\ 
        \bottomrule 
    \end{tabular}
\end{table}

\begin{table}[htbp] 
    \centering 
    \caption{Performance comparison of different imputation and interpolation methods against the proposed adaptive approach ('Ours') under the 'GYO Missing' condition.}
    \label{tab:imputation_comparison_gyo_missing} 
    \begin{tabular}{lcc}
        \toprule 
        \textbf{Method} & \textbf{OA (\%)} & \textbf{F1 (\%)} \\
        \midrule 
        \multicolumn{3}{l}{\textit{Advanced Imputation Methods}} \\ 
        VAE             & 57.95 & 55.28 \\
        NNRW            & 61.60 & 54.84 \\
        MISS FOREST     & 60.84 & 48.89 \\
        MICE            & 63.08 & 52.07 \\
        Impute Former   & 60.53 & 53.52 \\
        \midrule 
        \multicolumn{3}{l}{\textit{Classical Interpolation Methods}} \\ 
        cubic           & 53.48 & 52.89 \\
        lagrange        & 48.26 & 42.17 \\
        Akima           & 62.27 & 48.75 \\
        PCHIP           & 57.64 & 56.44 \\
        \midrule 
        \textbf{Ours (LoRA Adapt)} & \textbf{70.51} & \textbf{65.27} \\ 
        \bottomrule 
    \end{tabular}
\end{table}

\begin{table}[htbp] 
    \centering 
    \caption{Indoor Positioning Performance (Error Metrics) of lstm under Various Missing Modality Conditions using Adaptive Low-Rank Compensation.}
    \label{tab:positioning_adaptation} 
    \resizebox{\textwidth}{!}{
    \begin{tabular}{lcccc}
        \toprule 
        \textbf{Condition} & \textbf{Median Error} & \textbf{Mean Error} & \textbf{Param. Increase} & \textbf{Perf. Improve. (\%)} \\ 
        \midrule 
        Exp (no missing, lstm) & 0.6834 & 1.1651 & - & - \\ 
        \midrule 
        WIFI Missing       & 0.7475 & 1.1655 & \textbf{0.1754\%} & \textbf{12.91\%} \\
        IMU Missing        & 1.2121 & 1.5782 & \textbf{0.1754\%} & \textbf{18.54\%} \\
        UWB Missing        & 1.0176 & 1.4142 & \textbf{0.1754\%} & \textbf{46.18\%} \\
        \midrule 
        WIFI+IMU Missing   & 1.2496 & 1.6648 & \textbf{0.1169\%} & \textbf{15.47\%} \\
        WIFI+UWB Missing   & 1.6495 & 1.9506 & \textbf{0.1169\%} & \textbf{15.98\%} \\
        IMU+UWB Missing    & 2.2215 & 2.3419 & \textbf{0.1169\%} & \textbf{37.62\%} \\
        \bottomrule 
    \end{tabular}%
    } 
\end{table}

As evidenced by the results in Table 4, our Adaptive Low-Rank Compensation mechanism, applied to the pre-trained GRAM-MAMBA model, consistently and significantly outperforms all baseline imputation methods. While advanced methods like VAE or ImputeFormer attempt sophisticated data generation, and simpler methods like MICE or interpolation fill gaps based on observed data, they fundamentally operate by creating synthetic data before or independent of the downstream task model's inference process. This contrasts sharply with our approach, which directly adapts the perception model itself to compensate for the missing information source.

The superior performance of our method underscores the limitations inherent in relying solely on data imputation. Imputation techniques can suffer from cumulative errors, require unavailable contextual data (especially interpolation methods), fail to capture complex non-linear relationships accurately, or introduce biases, ultimately hindering the performance of the subsequent HAR classification. Furthermore, imputation often requires separate pre-processing steps. In contrast, our Adaptive Low-Rank Compensation mechanism leverages the knowledge encoded in the pre-trained model and efficiently fine-tunes minimal parameters (<0.3\%) to adjust the model's internal representations and fusion logic in response to the specific modality absence. This direct, parameter-efficient model adaptation proves more effective than merely attempting to reconstruct the missing input signal.To demonstrate that our method can be applied to other tasks and other model architectures, we also implemented a compensation strategy in the indoor positioning wireless sensing task based on lstm, achieving an average performance improvement of 24.45\% when the number of parameters increased by less than 0.2\%.

Therefore, these results validate the effectiveness of the Adaptive Low-Rank Compensation strategy as a robust and parameter-efficient solution for maintaining high HAR performance in multimodal systems facing real-world sensor data unavailability, offering clear advantages over traditional data imputation approaches.

\subsection{Ablation Study on GRAM-MAMBA}
To isolate and understand the impact of the GRAM loss component, this ablation study contrasts the inter-modal correlations learned by GRAM-MAMBA when the $L_{\text{GRAM}}$ term is active versus when it is ablated from the training objective.

\begin{figure}
    \centering
    \includegraphics[width=1\linewidth]{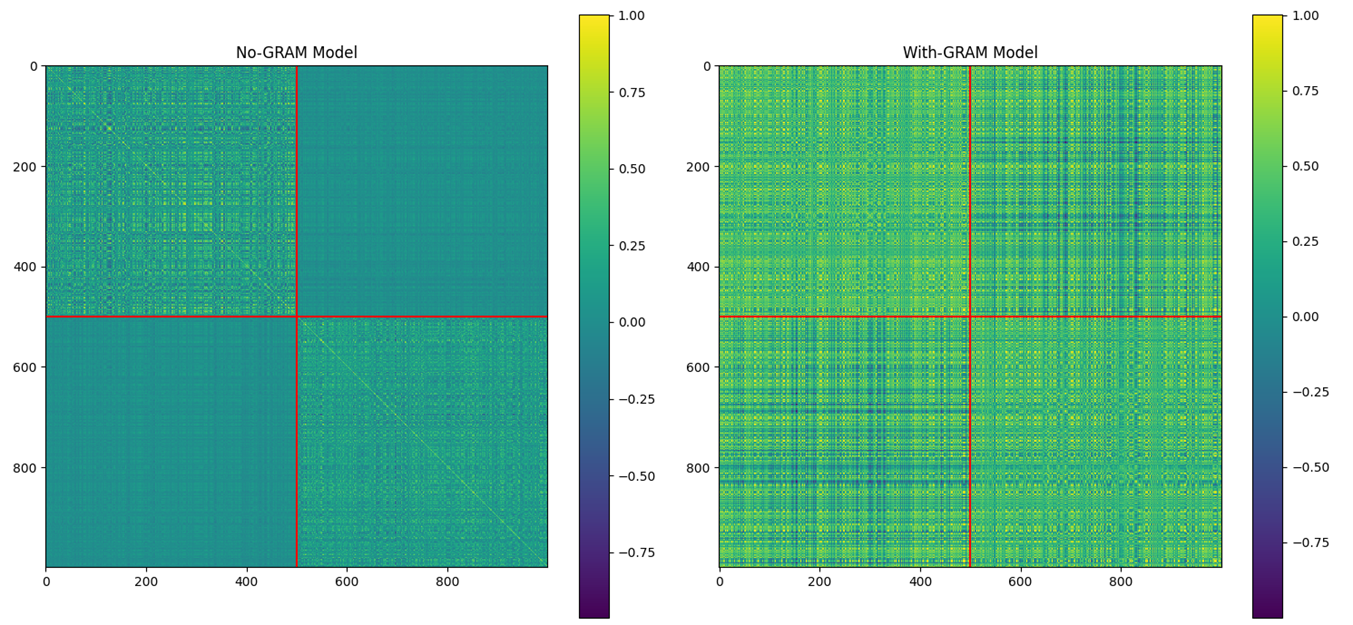}
    \caption{Gram Matrix Comparison: Inter-modal Correlations with and without GRAM Loss}
    \label{fig:gram_comparison}
\end{figure}
Figure~\ref{fig:gram_comparison} visually compares the inter-modal correlations via Gram matrices. The matrix from the model trained without the $L_{\text{GRAM}}$ term (left) shows weak off-diagonal correlations, particularly between the ACC and GYO modalities. In contrast, the matrix from the model including the $L_{\text{GRAM}}$ term (right) exhibits significantly stronger correlations in the corresponding off-diagonal blocks (representing ACC-GYO interaction). This demonstrates that optimizing with $L_{\text{GRAM}}$ successfully enhances feature correlation and alignment between different modalities for the same sample.

\section{DISCUSSION}
The GRAM-MAMBA framework represents a significant step towards realizing efficient, robust, and holistically aligned multi-modal perception systems suitable for practical IoT and wireless sensing applications. Its significance stems from concurrently addressing key limitations often tackled in isolation: the computational burden of advanced fusion models, suboptimal feature alignment, and the lack of efficient adaptation to sensor failures. By leveraging the Mamba architecture, GRAM-MAMBA achieves linear complexity scaling, making the processing of long sensor time-series feasible on resource-constrained edge devices where quadratic-complexity models often falter. Critically, it moves beyond simplistic concatenation or unidirectional alignment by incorporating an optimized GRAM matrix-based strategy; minimizing the Gram determinant explicitly encourages holistic, pairwise feature alignment across all modalities, capturing richer inter-dependencies and leading to more semantically cohesive representations, as validated by ablation studies. Furthermore, the introduced Adaptive Low-Rank Compensation mechanism, inspired by LoRA, provides a pragmatic and parameter-efficient (<0.3\% parameter updates) solution for robustness against missing modalities, demonstrably outperforming traditional data imputation techniques and avoiding computationally prohibitive retraining protocols. This unification of architectural efficiency, principled alignment, and adaptive robustness within a single framework offers substantial advantages for deploying reliable real-world systems for tasks like HAR and indoor positioning. The success of this approach points towards promising future directions, including its application to other multi-modal tasks, further optimization of the adaptation strategy, and hardware-specific implementations, ultimately paving the way for more dependable and capable multi-modal wireless perception in dynamic environments.

\section{Conclusion}
This paper presented GRAM-MAMBA, tackling critical efficiency, alignment, and adaptation challenges in multi-modal wireless perception. It uniquely combines Mamba's linear complexity, GRAM-based holistic alignment, and parameter-efficient (<0.3\%) LoRA-inspired adaptation for robustness against sensor failure. Achieving state-of-the-art accuracy and significantly outperforming imputation methods during adaptation, GRAM-MAMBA provides a unified, practical framework for building dependable multi-modal wireless systems suitable for dynamic, resource-constrained environments.

\bibliographystyle{ACM-Reference-Format}
\bibliography{gram_mamba_refs}

\end{document}